\documentclass{article}

\usepackage{microtype}
\usepackage{graphicx}
\usepackage{subcaption}
\usepackage{booktabs} 

\usepackage{hyperref}

\usepackage[accepted]{icml2026}
\setcitestyle{numbers,square,sort&compress,comma}

\usepackage{amsmath}
\usepackage{amssymb}
\usepackage{mathtools}
\usepackage{amsthm}

\usepackage{algorithm}
\usepackage{algpseudocode}
\usepackage{makecell}

\usepackage[capitalize,noabbrev]{cleveref}

\usepackage{caption}
\theoremstyle{plain}
\newtheorem{theorem}{Theorem}[section]

\theoremstyle{definition}

\theoremstyle{remark}
\newtheorem{remark}[theorem]{Remark}

\usepackage[textsize=tiny]{todonotes}

\begin{document}

\twocolumn[
\icmltitle{Detecting Adversarial Data via Provable Adversarial Noise Amplification}

\vskip 0.15in

\begin{center}
{\fontsize{11.25pt}{13.5pt}\bfseries\selectfont
Furkan Mumcu,
Yasin Yilmaz
}

{\fontsize{11.25pt}{13.5pt}\selectfont
University of South Florida
}

{\tt\small
\{furkan, yasiny\}@usf.edu
}
\end{center}

\vskip 0.25in
]



\printAffiliationsAndNotice{}  

\begin{abstract}

The nonuniform and growing impact of adversarial noise across the layers of deep neural networks has been used in the literature, without a formal mathematical justification, to detect adversarial inputs and improve robustness. In this work, we study this phenomenon in detail and present a formal  adversarial noise amplification theorem. We specify a set of sufficient conditions under which the adversarial noise amplification is mathematically guaranteed. Based on theoretical observations, we propose a novel training methodology with a custom spectral loss function and a specific architectural design to enhance the amplification signal for detecting adversarial data. Finally, we introduce a new, lightweight detection mechanism that leverages the enhanced  amplification signal and operates entirely at inference time. To validate our approach, we demonstrate the detector's efficacy against both state-of-the-art attacks and a purpose-built adaptive attack, confirming that enhanced amplification can serve as a robust and reliable signal for adversarial defense.

\end{abstract}
\section{Introduction}

Deep Neural Networks (DNNs) are known to be vulnerable to adversarial machine learning (ML) attacks, which manipulate the input data with small, carefully crafted noise to 
cause misclassification \cite{fgsm, pgd, bim, mumcu2022adversarial, aa, vmi_vni}. In response to the rapid growth in adversarial ML research, several defense strategies have been proposed. These can be broadly categorized into methods that alter the input to remove or reduce the noise's effect, efforts to develop inherently robust DNNs, and 
detectors to identify and reject adversarially altered data.

Among these defense strategies, several works have focused on leveraging the nonuniform impact of adversarial noise on network layers. 
\cite{dknn} uses it to increase DNN robustness to adversarial data. 
\cite{lr} recently introduced
a detector, Layer Regression (LR), 
based on this principle. The effectiveness of LR, empirically demonstrated across various models and data domains, was built upon a central conjecture: \textit{the impact of adversarial noise is minimal at early layers and magnified at later layers.}

While this conjecture is supported by extensive empirical evidence, it lacks a theoretical guarantee that could 
provide a deeper insight and 
strengthen its validity. 
The primary contribution of this work is to elevate this conjecture to a formal 
\emph{adversarial noise amplification
theorem}. 
By presenting sufficient conditions under which the theorem holds,
we solidify the theoretical underpinnings of any defense mechanism that benefits from this property and contribute to the broader understanding of how adversarial data propagates through DNNs.

Furthermore, we propose 
a specific architectural design using Leaky ReLU and a custom spectral loss function to construct a DNN with enhanced adversarial noise amplification signal. We provide empirical evidence showing that standard, well-trained networks often exhibit the behavior predicted by our theorem, and demonstrate that our new training method can enforce this property more rigorously. This work serves as a constructive proof that networks can be designed to have provable amplification properties, providing a reliable signal for detection.

\begin{figure}[t]
    \centering
    \includegraphics[width=0.99\linewidth]{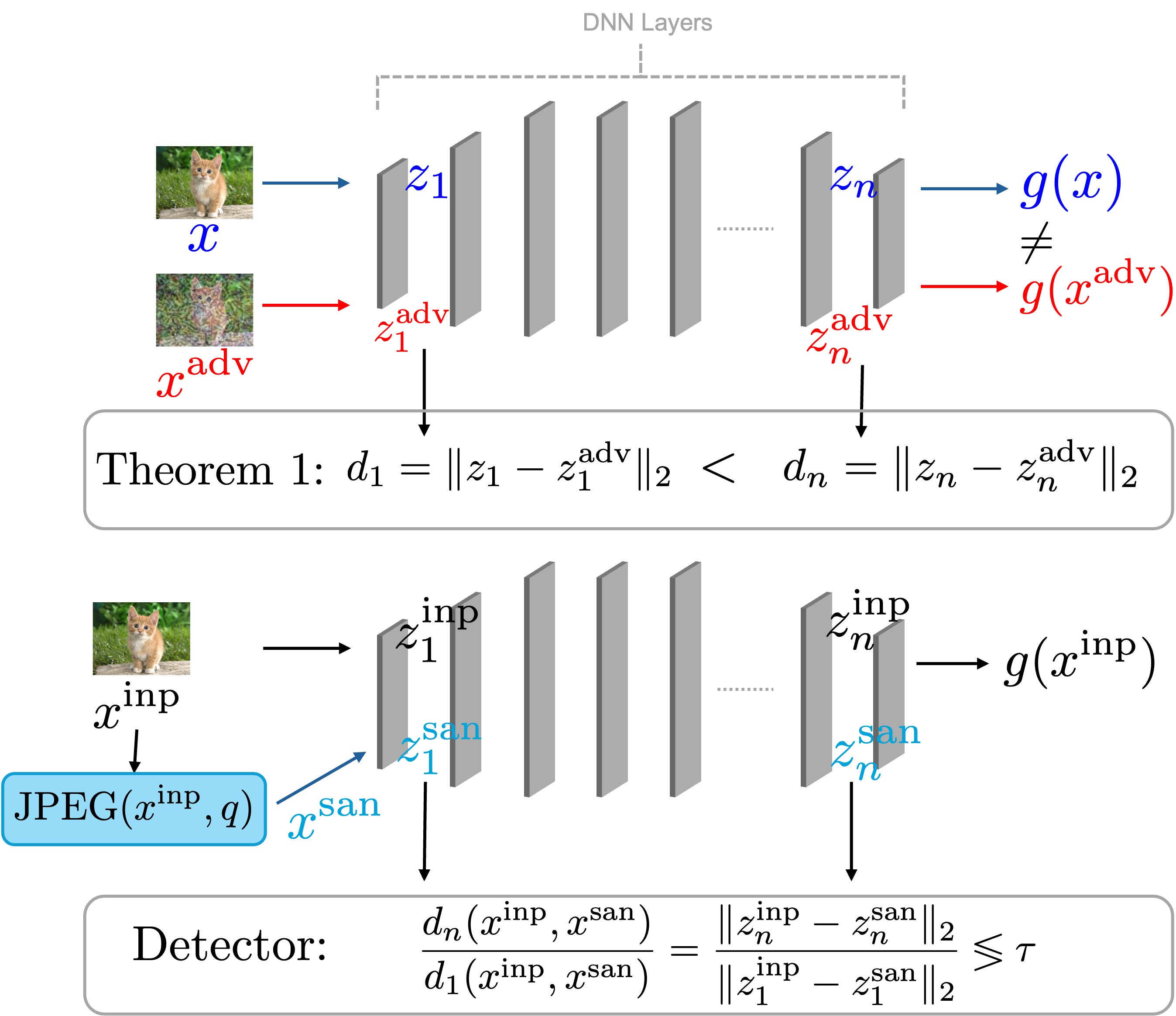}
  \caption{Adversarial noise amplification theorem (top) inspires the JPEG compression-based lightweight detector (bottom), which measures the amplification ratio with respect to the sanitized version $x^{\text{san}}$ of the input $x^{\text{inp}}$. If the ratio is greater than a predetermined threshold, then the input sample is labeled as adversarial and the predicted label $g(x^{\text{inp}})$ is ignored.}
    \label{fig:intro}
\end{figure}

Building upon this theoretical foundation, we also propose and evaluate a novel detection mechanism that directly leverages the amplification property (Figure \ref{fig:intro}).
Specifically, we operationalize it by combining our amplification analysis with the well-established defense technique of input sanitization. The proposed detector
measures the amplification signal by comparing an input image with
its sanitized version after applying a transformation designed to disrupt structured, high-frequency noise. To validate its efficacy, we benchmark the detector against a suite of state-of-the-art adversarial attacks. Importantly, we also evaluate it against a specially crafted adaptive attack aware of the detection method based on data sanitization and show that our detector still achieves a very high detection rate, confirming its practical value as a defense strategy. 

Importantly, we show that existing detection frameworks already achieve near-perfect detection performance against static attacks, a trend commonly observed among state-of-the-art methods. However, these approaches are typically evaluated only under standard threat models, leaving their robustness to adaptive attacks largely unexplored. We provide an extensive adaptive attack analysis for our framework, including direct comparisons with existing methods, and demonstrate that our detector maintains its performance even under worst-case adaptive attack scenarios. Our contributions can be summarized as follows:
\vspace{-3mm}
\begin{itemize}
    \item We formalize and provide a rigorous proof for the \emph{adversarial noise amplification theorem}, establishing a sufficient set of conditions under which the amplification of adversarial noise is guaranteed.\vspace{-2mm}
    \item We empirically demonstrate that the amplification phenomenon described by our theorem holds for existing, well-trained networks, and we propose 
    an architectural design and a new loss function 
    to construct a DNN that  satisfies the theorem's assumptions.\vspace{-2mm}
    \item We introduce a new, lightweight detection mechanism that 
    combines adversarial noise amplification with an input sanitization technique,
    providing a practical application for our theoretical findings.\vspace{-2mm}
    \item We conduct a thorough empirical evaluation, validating our detector's effectiveness against both state-of-the-art attacks and a purpose-built adaptive attack to demonstrate its robustness.
\end{itemize}

\section{Adversarial Noise Amplification}

\subsection{Preliminaries}

We first define the necessary components that will be used throughout the proof.

\textbf{Target Network Model:} We model an $n$-layer feed-forward neural network as a function $g: \mathbb{R}^d \to \mathbb{R}^c$, defined by the sequential function composition (denoted by $\circ$) of its layers: $g(x) = a_n \circ \dots \circ a_1(x)$. Each function $a_i$ represents the complete layer operation, which transforms an input $z$ from the preceding layer. This operation consists of two stages. First, an affine transformation is applied, which is parameterized by the layer's learnable weight matrix ($W_i$) and bias vector ($b_i$). Second, the result is passed through a pointwise, non-linear activation function ($f_i$). The complete layer operation is thus defined as $z_i=a_i(z_{i-1}) = f_i(W_i z_{i-1} + b_i)$, $z_0=x$. The linear transformation term, $W_i z_{i-1}$, serves as a general mathematical abstraction. This model is broadly applicable, encompassing the operations of fully-connected layers, the kernels in convolutional layers, and the individual weight matrices within the attention and MLP blocks of Transformers.

\textbf{Adversarial Data:} An adversarial 
    version
    $x^{\text{adv}}$ 
    of clean sample $x$
    is generated to 
    maximize the target network's loss function 
    $\mathcal{L}(g(x^{\text{adv}}), g(x))$ subject to the constraint $\|x^{\text{adv}} - x\|_{2} \le \epsilon$.

    
\textbf{Layer Impact:} The impact at layer $i$ is the change in its output vector, defined as $d_i(x^{\text{adv}},x) = \|a_i(\dots a_1(x^{\text{adv}})) - a_i(\dots a_1(x))\|_{2}=\|z_i^{\text{adv}}-z_i\|_2$.

\subsection{Theorem and the Proof}
In \cite{lr}, with a monotonicity assumption on the target loss function $\mathcal{L}(g(x^{\text{adv}}), g(x))$ with respect to the error norm $\|g(x^{\text{adv}})-g(x)\|$, it was conjectured and empirically shown that \emph{the impact of an adversarial sample is higher on the final layer output $d_n$ than the first layer output $d_1$}.

In the following theorem, we present sufficient conditions for the amplification of such adversarially designed noise throughout the target network. Later, in Section \ref{sec:model}, we leverage these sufficient conditions to enhance the amplification signal of adversarial noise to facilitate its detection. And in Section \ref{sec:detector}, we propose a practical adversarial sample detector based on this enhanced amplification signal. 

\begin{theorem}[Adversarial Noise Amplification Theorem] \label{thm:amplification}
An adversarial sample $x^{adv}$, which is generated to maximize $\mathcal{L}(g(x^{\text{adv}}), g(x))$, and in turn $\|g(x^{\text{adv}})-g(x)\|_2=\|a_n(z_{n-1}^{adv})-a_n(z_{n-1}))\|_2=d_n$ through the monotonicity assumption, impacts the final layer more than the first layer, i.e., $d_n>d_1$, under the following assumptions:
\begin{enumerate}
    \item 
    One-to-one Activation Function: Each activation function $f_i$ satisfies $\|f_i(u)-f_i(v)\|\not=0$ for $u\not=v$, i.e., there exists a constant $L_{f_i}>0$ such that $\|f_i(u) - f_i(v)\|_{2} \ge L_{f_i} \|u - v\|_{2}$. 
    
    \item
    Net Amplification Lower-Bound: The product of the minimum amplification factors for layers 2 through $n$ is
    greater than 1. This factor is determined by the
    activation function's expansion lower bound $L_{f_i}$
    and the smallest singular value 
    of the weight matrix $\sigma_{\min}(W_i)$:
    \[
    \beta = \prod_{i=2}^{n} L_{f_i} \sigma_{\min}(W_i) > 1
    \]
    
    \item {Non-Trivial Attack 
    :}
    Adversarial samples result
    in a nonzero impact at the first layer, i.e., $d_1 > 0$.
\end{enumerate}

\end{theorem}

\begin{proof}

The impact at layer $i$ is defined as
\[
    d_i = \|z_i^{\text{adv}} - z_i\|_{2}.
\]
Substituting the full definition of the layer operation $a_i$ for both the adversarial and clean inputs we get
\[
d_i = \|f_i(W_i z_{i-1}^{\text{adv}} + b_i) - f_i(W_i z_{i-1} + b_i)\|_{2}.
\]

From Assumption 1, 
\[
d_i \ge L_{f_i}  \|(W_i z_{i-1}^{\text{adv}} + b_i) - (W_i z_{i-1} + b_i)\|_{2},
\]

\[
d_i \ge L_{f_i}  \|W_i (z_{i-1}^{\text{adv}} - z_{i-1})\|_{2}.
\]
Since for a linear transformation $W_i$, the output vector's norm is lower-bounded by the smallest singular value $\sigma_{\min}(W_i)$ times the norm of the input vector: 
$\|W_i x\| \ge \sigma_{\min}(W_i) \|x\|$,
\[
d_i \ge L_{f_i} \cdot \sigma_{\min}(W_i) \cdot \|z_{i-1}^{\text{adv}} - z_{i-1}\|_{2}.
\]

\begin{equation} \label{eq:recursive_lower}
d_i \ge L_{f_i} \cdot \sigma_{\min}(W_i) \cdot d_{i-1} 
\end{equation}

Using the recursive inequality \eqref{eq:recursive_lower}  
we can relate the final impact $d_n$ to the initial impact $d_1$ by chaining the lower bounds. For the final layer $n$, we have
\[
d_n \ge L_{f_n} \cdot \sigma_{\min}(W_n) \cdot d_{n-1}.
\]
Substituting the bound for $d_{n-1}$ we get
\[
d_n \ge L_{f_n} \cdot \sigma_{\min}(W_n) \cdot L_{f_{n-1}} \cdot \sigma_{\min}(W_{n-1}) \cdot d_{n-2}.
\]
Continuing this process until we reach $d_1$, we obtain the final lower bound for $d_n$ in terms of $d_1$:
\begin{equation} \label{eq:chained_lower}
d_n \ge \left( \prod_{i=2}^{n} L_{f_i} \cdot \sigma_{\min}(W_i) \right) \cdot d_1. 
\end{equation}

From Assumption 2, 

$d_n \ge \beta \cdot d_1 > d_1.$

Assumption 3 completes the proof that

$d_n > d_1.$
\end{proof}

\begin{remark}
While our theorem is proven for the $\ell_2$-norm, the principle of norm equivalence provides a strong theoretical basis for expecting a similar amplification phenomenon in the $\ell_\infty$ case. The inequality $\|v\|_{2} \le \sqrt{dim} \|v\|_{\infty}$, where $dim$ is the vector dimension, ensures that any $\ell_\infty$ perturbation has a corresponding bounded $\ell_2$ counterpart. Our proof demonstrates that the target network is a guaranteed amplifier of this $\ell_2$ norm. Thus, our theorem describes the fundamental mechanism by which $\ell_\infty$ perturbations are also amplified. 
\end{remark}

\subsection{Discussion of the Assumptions}
\label{sec:discussion_of_assumptions}

Our theorem is built upon three assumptions that formalize the properties of the network's activations, its weights, and the nature of the attack. Here, we discuss the validity and implications of each premise.

\textit{One-to-one Activation Function.} This assumption requires that each activation function possesses an expansion lower-bound greater than zero ($L_{f_i} > 0$). 
The validity of this condition in practice depends on the specific activation function. For the {ReLU} function, this property does not hold in general 
as it outputs zero for all negative inputs.
In contrast, other activations like Leaky ReLU, tanh, and logistic sigmoid function satisfy this 
condition as their minimum slope is a small positive constant for finite input.

\textit{Net Amplification Lower-Bound.} 
With $L_{f_i}>0$ by choosing a suitable activation function, this is an assumption on the minimum singular values of weight matrices. While $\prod_{i=2}^n L_{f_i} \sigma_{\min}(W_i)>1$ is a strict assumption, this is a sufficient, not a necessary, condition. In practice, this condition implies that $\sigma_{\min}(W_i)$ should to be controlled to enhance adversarial noise amplification. Moreover, such singular value regularization is known to be useful for gradient conditioning and training stability \cite{jia2017improving}, improved robustness to adversarial noise \cite{senderovich2022towards}, and information preservation \cite{zaeemzadeh2020norm}, which are common features of well-trained and effective DNNs. 

\textit{Non-Trivial Attack Condition.} This is a mild, technical assumption that simply requires the initial impact of the attack to be non-zero ($d_1 > 0$). It serves to exclude the trivial case where the perturbation has no effect on the first layer's output. This is inherently true for any meaningful adversarial attack since the goal is to generate a perturbation that has a cascading effect through the network.

\section{Enhancing  Amplification Signal}
\label{sec:model}

Empirical analysis in Section \ref{sec:exp} demonstrates that the amplification phenomenon 
typically
arises in existing, well-trained networks 
with adversarial noise, $d_n/d_1>1$, but not with natural noise.  
In this section, using the sufficient conditions presented in Theorem \ref{thm:amplification}, we enhance the amplification signal $d_n/d_1$ of adversarial noise to facilitate the detection of adversarial samples. 

To satisfy Assumption 1, we adopt the Leaky ReLU activation function,
\begin{equation}
    f_i(x) =
    \begin{cases}
      x & \text{if } x \ge 0,\\
      \alpha x & \text{if } x < 0,
    \end{cases}
\end{equation}
where $\alpha$ is a small positive constant (e.g., 0.01). Because its minimum slope is $\alpha$, the expansion lower-bound constant is $L_{f_i}=\alpha>0$. 

Motivated by Assumption 2, we introduce a custom loss function that combines the standard classification loss $\mathcal{L}$ with a spectral regularization term,
which promotes larger $\sigma_{\min}(W_i)$ to obtain a net amplification greater than one:
\begin{equation}
    \mathcal{L}_{\text{total}} = \mathcal{L} - \lambda \sum_{i=2}^{n} \log(\sigma_{\min}(W_i)),
\end{equation}
The logarithm provides numerical stability, and the negative sign allows maximization of $\sigma_{\min}(W_i)$ within a minimization framework. The hyperparameter $\lambda$ balances the trade-off between classification accuracy and spectral amplification.

While this approach enforces the amplification property during training, the resulting guarantee is statistical rather than absolute.

\section{Detection via Amplification Signal}
\label{sec:detector}

The Adversarial Noise Amplification Theorem provides a theoretical foundation for using the net amplification ratio $d_n/d_1$ as a measurable signal of adversarial behavior. 
Note that since naturally occurring non-adversarial noise is not generated to maximize $d_n$, it is not necessarily amplified through the network layers. Experimental results provided in Table \ref{tab:res_random_noise} corroborates this statistical expectation: $d_n/d_1$ is much higher than $1$ for adversarial noise, but less than $1$ for natural noise. 
Building on this principle, we propose a practical detection mechanism called the JPEG-based Amplification Detector (JAD). 

Given the input sample $x^{\text{inp}}$, we obtain its sanitized version $x^{\text{san}}$ through JPEG compression to compute the amplification signal
\begin{align}
    \text{JAD\_score}(x^{\text{inp}}) &= \frac{d_n(x^{\text{inp}}, x^{\text{san}})}{d_1(x^{\text{inp}}, x^{\text{san}})}, \\
    x^{\text{san}} &= \mathrm{JPEG}(x^{\text{inp}}, q), \\
    d_i(x^{\text{inp}}, x^{\text{san}}) &= \| z_i^{\text{inp}} - z_i^{\text{san}} \|_{2},
\end{align}
where $q$ denotes the JPEG quality parameter. 
JAD measures how strongly the network amplifies the difference introduced by such sanitization.
An input is labeled as adversarial if its JAD score exceeds a detection threshold $\tau$, determined empirically from validation data.

Because JPEG selectively suppresses high-frequency components, this transformation attenuates structured adversarial noise while leaving benign content relatively stable. The resulting change $d_i(x^{\text{inp}}, x^{\text{san}}) = \| z_i^{\text{inp}} - z_i^{\text{san}} \|_{2}$ mimics $d_i = \|z_i^{\text{adv}} - z_i\|_{2}$ for adversarial inputs with a statistically distinct signature.
For clean samples, the effect of sanitization resembles random, unstructured compression artifacts, which the network suppresses or amplifies only weakly, resulting in a small JAD score. To further increase robustness against adaptive attacks, the JPEG quality parameter $q$ can be randomized within a predefined range during inference, 
as studied in Section \ref{sec:adaptive}. 

A crucial advantage of JAD is its simplicity and self-contained nature. It does not require auxiliary models, stored clean references, adversarial training samples, or layer-wise regression modules. Instead, it captures layer amplification by measuring how sanitization differentially affects early and late layer activations. The detector operates entirely at inference time, making it lightweight and easily deployable.

\section{Empirical Analysis and Experiments}
\label{sec:exp}

In this section, we present an extensive empirical analysis to validate our theoretical framework, training methodology, and the proposed JAD detector, with additional implementation details provided in Appendix \ref{app:code}. Our evaluation includes: (i) verifying that the amplification phenomenon predicted by our theorem naturally appears in standard deep networks, (ii) demonstrating that our training method strengthens this amplification while preserving accuracy, (iii) benchmarking JAD against state-of-the-art detection methods across multiple attacks, and (iv) assessing its robustness under a tailored adaptive attack that explicitly accounts for the defense.

\subsection{Amplification in Existing DNNs}
\label{sec:amp_existing}

\begin{table}[t]
\caption{Amplification results for existing DNNs. Each cell shows the Success Rate (Average Amplification, $d_n/d_1$) for a given model and attack.}
\small
\centering
\resizebox{\columnwidth}{!}{%
\begin{tabular}{l|ccccc}
\hline
\textbf{Model} & \textbf{PGD} & \textbf{BIM} & \textbf{VNI} & \textbf{VMI} & \textbf{PGD-{$\ell_2$}} \\ \hline
ResNet-50 & 1.0 (48.59) & 1.0 (47.94) & 1.0 (31.14) & 1.0 (36.26) & 1.0 (88.69) \\
Inception-v3 & 1.0 (294.60) & 1.0 (297.67) & 1.0 (259.20) & 1.0 (321.41) & 1.0 (596.69) \\
ViT & 1.0 (102.32) & 1.0 (95.01) & 1.0 (55.14) & 1.0 (61.54) & 1.0 (187.50) \\
DeiT & 1.0 (49.18) & 1.0 (45.63) & 1.0 (28.91) & 1.0 (32.83) & 1.0 (135.96) \\ 
\hline
\end{tabular}}

\label{tab:res_1}
\end{table}

\begin{figure}[t]
    \centering
    \includegraphics[width=0.65\linewidth]{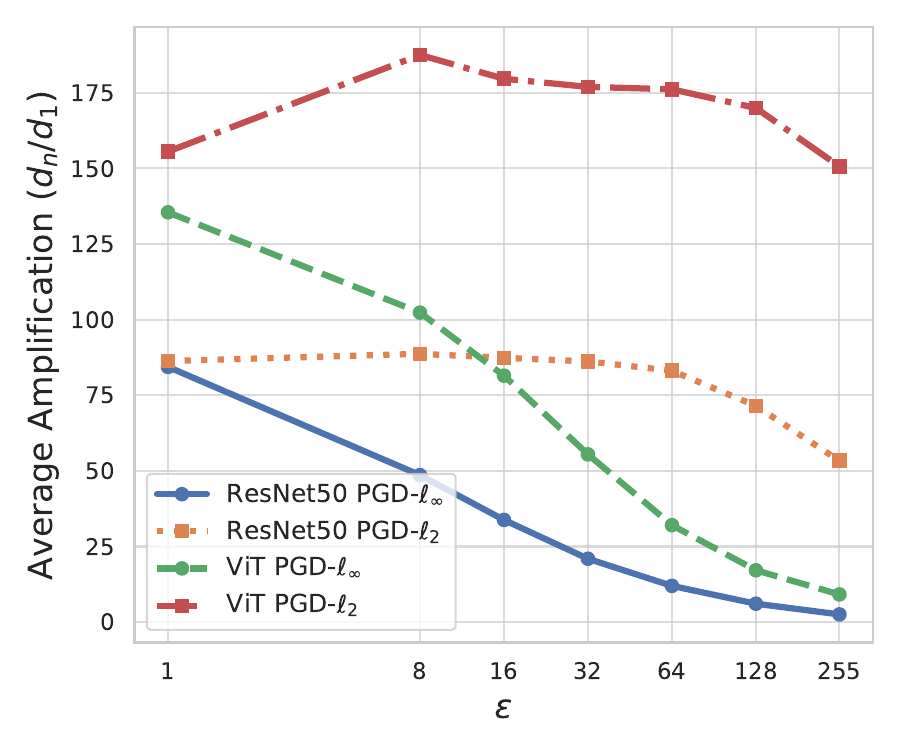}
  \caption{Amplification vs. adversarial noise budget $\epsilon$ for different models and attacks.}
    \label{fig:eps_vs_amp}
\end{figure}

As discussed in Section~\ref{sec:discussion_of_assumptions}, well-trained DNNs are expected to inherently exhibit the amplification behavior described by our theorem. To verify this, we evaluate four widely used architectures from the {timm} library \cite{timm}, namely ResNet-50 \cite{resnet}, Inception-v3 \cite{inception}, ViT \cite{vit}, and DeiT \cite{deit}, on 1,000 randomly selected images from the ImageNet \cite{imagenet} validation set. For each model, we compute the net amplification ratio ($d_n/d_1$) under strong gradient-based attacks that successfully induce misclassification, including PGD~\cite{pgd}, BIM~\cite{bim}, VMI~\cite{vmi_vni}, and VNI~\cite{vmi_vni}, using the default configurations recommended in their respective works. 

Table~\ref{tab:res_1} summarizes the results for each model and attack. For every combination, we report two complementary measures: the \textit{Success Rate}, which denotes the proportion of test images where the amplification ratio $d_n/d_1$ exceeds one, and the \textit{Average Amplification}, representing the mean value of this ratio across all samples.

The results reveal a consistent pattern. Across all models and attacks, 100\% of adversarial examples exhibit $d_n/d_1 > 1$, confirming that standard networks already amplify adversarial perturbations. The amplification magnitudes are substantial, often ranging from tens to hundreds, demonstrating that the effect is both pervasive and strong. Variations across architectures are observed, with Inception-v3 showing the highest average amplification. Moreover, PGD constrained by the $\ell_2$ norm consistently yields the strongest amplification, aligning with the theoretical $\ell_2$ basis of our formulation.

To further explore how this behavior evolves with attack strength, we vary the perturbation budget $\epsilon$ and measure the resulting amplification. Figure~\ref{fig:eps_vs_amp} shows the average ratio for ResNet-50 and ViT under PGD $\ell_\infty$ and $\ell_2$ attacks. As $\epsilon$ increases, the amplification ratio ($d_n/d_1$) decreases, most notably for $\ell_\infty$ attacks. This occurs because strong perturbations already induce large changes in the early layers, leaving less residual signal to be amplified by deeper layers. At such high perturbation levels, the images become heavily distorted and no longer represent subtle adversarial examples, which are the primary target of an attacker. In contrast, small and structured perturbations produce limited but coherent feature-level deviations that propagate and magnify through the network, resulting in stronger amplification. The relatively stable trend for $\ell_2$ attacks further indicates that the network’s inherent amplification dynamics align closely with the $\ell_2$ geometry underlying our theoretical framework.


\subsection{Performance of Proposed Network with Enhanced Amplification}
\label{sec:5.2}

We train a DNN with the proposed activation function and loss function described in Section~\ref{sec:model} to empirically validate our methodology. Two versions of the WideResNet \cite{wresnet} architecture are trained on CIFAR-10 \cite{cifar} for 50~epochs: a \textit{Default} configuration using standard training and an \textit{Ours} configuration incorporating our amplification-enforcing design. To ensure statistical robustness, each experiment is repeated 100~times with independent initializations.

Our goal is to confirm that the proposed modifications strengthen the amplification property without degrading classification accuracy. Evaluated on the CIFAR-10 test set, the proposed model achieved an accuracy of $87.21\% \pm 0.05$, compared to $86.67\% \pm 0.05$ for the default model, indicating that our approach preserves standard predictive performance.

We then analyze the amplification characteristics of both networks under the five adversarial attacks, with the same attack configurations as in Section \ref{sec:amp_existing}, on the CIFAR-10 test set. The results in Table~\ref{tab:res_ours_vs_default} show that while both models consistently exhibit the amplification phenomenon (Success Rate~$=1.0$), our training substantially increases the magnitude of the amplification ratio. Across all attacks, the average amplification achieved by our model is notably higher than that of the default configuration. These results verify that the proposed spectral regularization and activation function effectively enforce the theoretical conditions of our theorem, yielding a network with a stronger amplification signal.

\begin{table}[t]
\caption{Amplification comparison between WideResNet models trained with our proposed structure and loss function versus the default configuration. Each cell reports \textit{Success Rate (Average Amplification)} for the amplification ratio $d_n/d_1$.}
\scriptsize
\centering
\resizebox{\columnwidth}{!}{%
\begin{tabular}{l|ccccc}
\hline
\textbf{Method} & \textbf{PGD} & \textbf{BIM} & \textbf{VNI} & \textbf{VMI} & \textbf{PGD\textsubscript{$\ell_2$}} \\ \hline
Ours    & 1.0 (99.88) & 1.0 (101.12) & 1.0 (79.97) & 1.0 (84.52) & 1.0 (93.78) \\
Default & 1.0 (84.99) & 1.0 (87.74)  & 1.0 (68.77) & 1.0 (73.48) & 1.0 (86.97) \\ \hline
\end{tabular}}

\label{tab:res_ours_vs_default}
\end{table}

\subsection{Specificity of the Amplification Phenomenon}
\label{app:specificity}

\begin{table}[t]
\caption{Comparison of amplification results for adversarial noise and non-adversarial image corruptions on our trained WideResNet. {Success Rate} denotes the proportion of test images for which $d_n/d_1 > 1$.}
\footnotesize
\centering
\begin{tabular}{l|cc}
\hline
\textbf{Corruption Type} & \makecell{\textbf{Success} \\ \textbf{Rate}} & \makecell{\textbf{Avg.} \\ \textbf{Amplification}} \\ \hline
PGD ($\ell_\infty$)    & 1.000                 & 99.88                       \\
PGD $\ell_2$              & 1.000                 & 93.78                       \\ \hline \hline
$\ell_\infty$ Uniform  & 0.076                 & 0.73                        \\
$\ell_2$ Gaussian         & 0.028                 & 0.53                        \\
Salt-and-Pepper     & 0.000                 & 0.36                        \\
Gaussian Blur       & 0.083                 & 0.75                        \\
JPEG Compression    & 0.072                 & 0.58                        \\
Laplacian Noise     & 0.000                 & 0.42                        \\ \hline
\end{tabular}

\label{tab:res_random_noise}
\end{table}

We next perform a set of control experiments with our trained WideResNet model on the CIFAR-10 The non-adversarial corruptions are selected to span diverse structural characteristics. To ensure a fair comparison, we first include $\ell_\infty$ uniform and $\ell_2$ Gaussian noise with magnitudes exactly matched to their respective PGD attacks. We then extend the analysis to a broader set of common corruptions including Salt-and-Pepper noise, Laplacian noise, Gaussian Blur, and JPEG compression, representing impulsive, heavy-tailed, low-frequency, and lossy transformations, respectively. This setup allows us to test whether amplification arises exclusively from the structured, optimized nature of adversarial perturbations.

The results in Table~\ref{tab:res_random_noise} strongly validate our hypothesis. The contrast between adversarial and non-adversarial perturbations is striking. A PGD $\ell_\infty$ attack yields a Success Rate of 1.0 and an Average Amplification of 99.88, whereas random uniform noise of the same $\ell_\infty$ magnitude produces only 0.076 and 0.73, respectively. The same pattern holds for the $\ell_2$ case. All other non-adversarial corruptions from impulsive Salt-and-Pepper noise to low-frequency Gaussian Blur, lead to net suppression rather than amplification. These results confirm that the amplification phenomenon is a specific and reliable signature of structured adversarial perturbations, not a general reaction to image degradation.

\subsection{Performance of JAD}
\label{sec:5.4}

\begin{table}[t]
\caption{Comparison of AUROC scores across different attacks for WideResNet using default (-D) and our (-O) configurations.}
\scriptsize
\centering
\begin{tabular}{l|ccccccc}
\hline
\textbf{Model} & \textbf{PGD} & \textbf{BIM} & \textbf{VNI} & \textbf{VMI} & \textbf{PGD-{$\ell_2$}} & \textbf{AA} & \textbf{Diff} \\ \hline
-D & 0.99 & 0.99 & 0.98 & 0.97 & 0.99 & 0.99 & 0.99 \\
-O & 0.99 & 0.99 & 0.99 & 0.98 & 0.99  & 0.99 & 0.99 \\ \hline
\end{tabular}

\label{tab:wideresnet_comparison}
\end{table}

\begin{table}[t]
\caption{Comparison of AUROC scores across different attacks for various detectors.}
\scriptsize
\centering
\begin{tabular}{l|ccccccc}
\hline
\textbf{Method} & \textbf{PGD} & \textbf{BIM} & \textbf{VNI} & \textbf{VMI} & \textbf{PGD-$\ell_2$} & \textbf{AA} & \textbf{Diff} \\ \hline
DkNN   & 0.86 & 0.84 & 0.80 & 0.79 & 0.85 & 0.80 & 0.81 \\
VLAD   & 0.83 & 0.85 & 0.89 & 0.90 & 0.81 & 0.75 & 0.80 \\
EPS-AD & 0.99 & 0.99 & 0.99 & 0.95 & 0.96 & 0.96 & 0.95 \\
LR     & 0.99 & 0.99 & 0.99 & 0.99 & 0.99 & 0.99 & 0.97 \\
JAD    & 0.99 & 0.99 & 0.99 & 0.98 & 0.99 & 0.99 & 0.98 \\ \hline\end{tabular}

\label{tab:methods_comparison}
\end{table}

In this section, we evaluate the performance of our proposed JPEG-based Amplification Detector (JAD). We use the area under the receiver operating characteristic (AUROC) curve as our evaluation metric, a standard measure of a detector’s discriminative capability. For evaluation, we construct a balanced test set consisting of 1,000 clean images from the CIFAR-10 test set and 1,000 corresponding adversarial examples that successfully induce misclassification. We compute AUROC per attack, each using 2,000 samples (1,000 clean + 1,000 adversarial), and repeat the process for five attack types: PGD \cite{pgd}, PGD-{$\ell_2$}, BIM \cite{bim}, VNI \cite{vmi_vni}, and VMI \cite{vmi_vni}, Diff-PGD \cite{diff}, and AutoAttack (AA) \cite{aa}. Our version of WideResNet from Section \ref{sec:5.2} used as target model.

Our first experiment demonstrates the benefit of using a network specifically trained to amplify adversarial noise. We evaluate JAD on both the standard WideResNet (Default) and our version trained with the proposed amplification-oriented methodology (Ours). The AUROC results, shown in Table~\ref{tab:wideresnet_comparison}, indicate that JAD achieves near-perfect detection across all attacks. Our configuration provides a slight but consistent improvement, particularly against momentum-based attacks (VNI and VMI), as the network is trained to produce a more stable and pronounced amplification signal, offering a stronger foundation for detection.

Next, we compare JAD with several state-of-the-art detectors, including Deep k-Nearest Neighbors (DkNN) \cite{dknn}, VLAD \cite{vlad}, EPS-AD \cite{eps-ad}, and Layer Regression (LR) \cite{lr}, all evaluated using our trained WideResNet model. As shown in Table~\ref{tab:methods_comparison}, JAD achieves near-perfect AUROC scores across all attack types, performing on par with EPS-AD and LR. DkNN and VLAD, while effective, show lower overall performance, particularly under momentum-based attacks, indicating reduced generalization to different perturbation dynamics. These results confirm that JAD provides a strong and consistent baseline under standard adversarial conditions. In Appendix \ref{app:jpeg_standalone}, in addition to comparing against state-of-the-art detectors, we also perform an ablation study examining JPEG sanitization as a standalone detector, showing that it performs poorly without JAD’s amplification component. In Appendix~\ref{app:imagenet}, we present additional results of the JAD detector on robust models and the ImageNet dataset, demonstrating consistent performance.

\subsection{JAD Performance against Adaptive Attack}
\label{sec:adaptive}

\begin{table}[t]
\caption{AUROC performance of JAD and Attack SR against an adaptive attack with varying $\epsilon$ values.}
\small
\centering
\resizebox{\columnwidth}{!}{%
\begin{tabular}{l|cccccc}
\hline
 & \textbf{$\epsilon$=4} & \textbf{$\epsilon$=8} & \textbf{$\epsilon$=16} & \textbf{$\epsilon$=32} & \textbf{$\epsilon$=64} & \textbf{$\epsilon$=128} \\ \hline
JAD AUROC   & 0.96 & 0.99 & 0.99 & 0.98 & 0.93 & 0.91 \\ 
ASR         & 0.95 & 1.0  & 1.0  & 1.0  & 1.0 & 1.0 \\ 
\hline
\end{tabular}}

\label{tab:adaptive_results}
\end{table}

To evaluate JAD under a fully informed adversary, we extend the standard 
PGD framework using the same experimental setup as in Section \ref{sec:5.4}, including the clean dataset, model, and evaluation protocol. Alongside detection AUROC, we also report the attack success rate (ASR), defined as the proportion of adversarial examples that successfully change the classifier’s prediction.
A conventional PGD attack seeks an adversarial example \(x^{\text{adv}}\) that maximizes the classifier loss subject to an \(\ell_\infty\) constraint.

However, a true adaptive adversary must also account for our detector. We first develop an EOT based formulation \cite{athalye2018synthesizing} adaptive attack in which the adversary tries to  maximize the classifier loss both before and after sanitization, aiming to suppress the amplification ratio $d_n/d_1$ that JAD relies on for detection.

\begin{equation}
\begin{aligned}
\max_{x^{\text{adv}}}\quad & (1-\lambda)\,\mathcal{L}_{\mathrm{CE}}(g(x^{\text{adv}}), y) \\
&\quad + \lambda\,\mathbb{E}_{q\sim U(Q)}\big[\mathcal{L}_{\mathrm{CE}}(g(\mathrm{JPEG}(x^{\text{adv}},q)), y)\big] \\
\text{s.t.}\quad & \|x^{\text{adv}}-x\|_{\infty}\le\epsilon,
\end{aligned}
\end{equation}
where \(\lambda\in[0,1]\) controls the trade-off between fooling the classifier and the detector, and $U(Q)$ denotes the uniform distribution over the range $Q$ of JPEG quality $q$ values. In practice, the expectation is approximated by averaging over \(T\) JPEG quality samples per PGD step. The attacker can tune three levers to strengthen the attack, the perturbation budget \(\epsilon\), the number of trials \(T\), and the trade off parameter \(\lambda\). Next we investigate the effect of these parameters.

Table~\ref{tab:adaptive_results} summarizes how JAD’s AUROC varies with the perturbation strength $\epsilon$ under the adaptive attack configured with $T=1$ and $\lambda=0.5$, a balanced setting in which the adversary attempts to fool both the classifier and the detector. JAD already performs strongly at small budgets ($\epsilon=4$, AUROC 0.96) and reaches peak accuracy for moderate perturbation levels ($\epsilon=8$ to $\epsilon=32$), where the adversarial noise remains structured and fully amplifiable, yielding AUROCs of 0.98--0.99. At very large budgets ($\epsilon=64,128$), AUROC decreases modestly to 0.93 and 0.91. As discussed in Section~\ref{sec:amp_existing}, such large perturbations create substantial early-layer impact and heavily distort the input, reducing the amplification ratio that JAD relies on. Even so, JAD maintains strong separability well beyond realistic adversarial threat levels. Finally, varying the number of JPEG trials $T$ from 1 to 20 produced only minor fluctuations (within approximately one percentage point), indicating that increasing $T$ has little influence on the success of the adaptive attack while significantly increasing the computations.

\begin{figure}[t]
    \centering
    \includegraphics[width=0.65\linewidth]{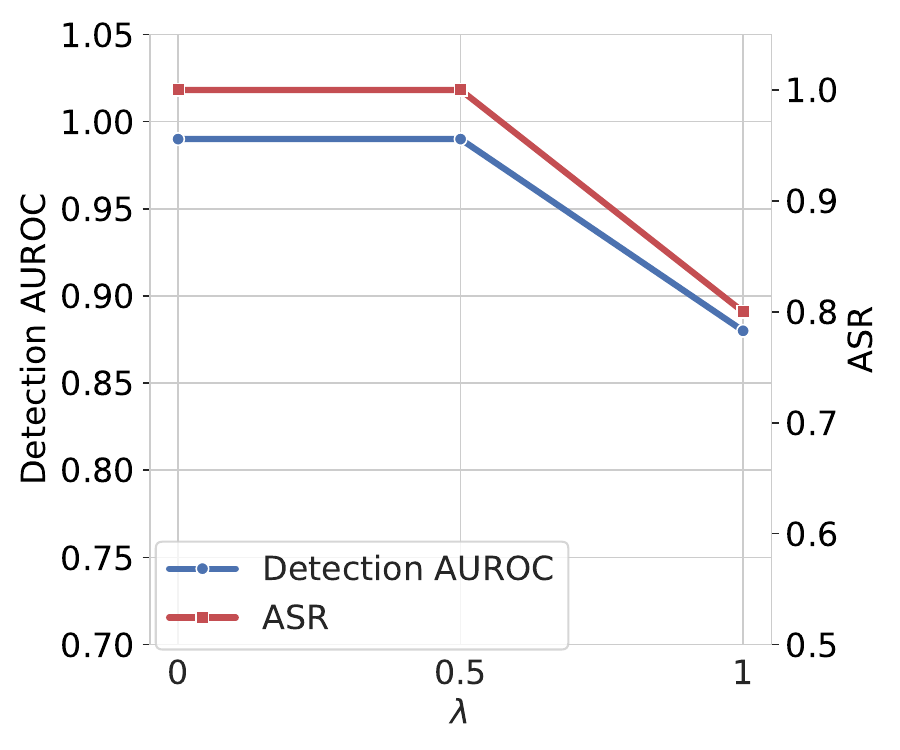}
  \caption{Trade-off between attack success rate and JAD’s detection AUROC as the adaptive attack shifts its objective via $\lambda$.}
    \label{fig:adaptive_params}
\end{figure}

Figure~\ref{fig:adaptive_params} shows how the adaptive attack behaves as the trade off parameter $\lambda$ varies from 0 to 1. When $\lambda=0$, the attacker focuses solely on fooling the classifier. This yields a high attack success rate, but the resulting perturbations retain strong amplifiable structure, allowing JAD to detect them with near perfect accuracy. Conversely, when $\lambda=1$, the attacker concentrates entirely on evading detection. While this reduces the detector’s AUROC, it also sharply lowers the attack success rate because the perturbation is no longer effective at misleading the classifier. These two extremes demonstrate the core limitation faced by adaptive adversaries: improving one objective necessarily harms the other. JAD’s randomized inference and amplification based signal force this incompatibility, preventing fully informed attackers from jointly achieving high attack success and low detectability.

We further design a classical PGD-based targeted adaptive attack, following a similar approach to existing works \citep{lr, contra} with the objective function
\begin{equation}
    \max_{x^{\text{adv}}} \mathcal{L}_{\text{Classifier}} - \lambda \cdot \mathcal{L}_{\text{Detector}}
\end{equation}
where $\mathcal{L}_{\text{Classifier}}$ and $\mathcal{L}_{\text{Detector}}$ denote the loss functions of the target model and JAD score ($d_n/d_1$), respectively, while $\lambda$ balances the two components. To provide a comparison, we extend the adaptive PGD framework to other state-of-the-art detectors, including DkNN, VLAD, EPS-AD, and LR. For VLAD, EPS-AD, and LR, the detector-specific loss $\mathcal{L}_{\text{Detector}}$ corresponds to the differentiable loss function of each method's detection model. For DkNN, we adopt the adaptive attack formulation proposed by its authors, defining $\mathcal{L}_{\text{Detector}}$ as the $\ell_2$ distance between the internal representation of the adversarial input at layer 1 and that of a target training sample. All adaptive attacks are performed with 200 PGD steps, $\lambda = 1$, and $\epsilon = 8$, consistent with the experimental settings reported in existing works \cite{lr, adaptive_setting}.

\begin{table}[t]
\caption{Comparison of detector performance against default PGD versus the unified adaptive attack. Where JAD-1 and JAD-2 reports the EOT-based and classical adaptive attacks against our method.}
\small
\centering
\begin{tabular}{l|c|c}
\hline
\textbf{Method} & \makecell{\textbf{AUROC} \\ \textbf{(Default $\rightarrow$ Adaptive)}} & \makecell{\textbf{Performance} \\ \textbf{Drop}} \\ \hline
DkNN   & 0.86 $\rightarrow$ 0.15 & 0.71 \\
VLAD   & 0.83 $\rightarrow$ 0.10 & 0.73 \\
EPS-AD & 0.99 $\rightarrow$ 0.36 & 0.63 \\
LR     & 0.99 $\rightarrow$ 0.45 & 0.54 \\ 
\hline
JAD-1    & 0.99 $\rightarrow$ 0.98 & 0.01 \\ 
JAD-2   & 0.99 $\rightarrow$ 0.99 & 0.00 \\ \hline
\end{tabular}

\label{tab:adaptive_comparison}
\end{table}

As shown in Table~\ref{tab:adaptive_comparison}, all competing methods experience a dramatic reduction in AUROC under the adaptive attack, confirming their vulnerability when the adversary has full knowledge of the detection mechanism. VLAD’s performance collapses from 0.83 to 0.10 (a drop of 0.73), EPS-AD falls from 0.99 to 0.36 (a drop of 0.63), and Layer Regression declines from 0.99 to 0.45 (a drop of 0.54). In contrast, JAD maintains AUROC scores of 0.98 and 0.99, exhibiting only a marginal drop of 0.01 under EOT-based attack and no degradation under classical adaptive attack, highlighting its robustness against fully informed adversaries. This robustness arises from two factors. First, JAD’s inference-time randomization prevents the attacker from obtaining stable gradients, forcing reliance on noisy repeated evaluations. Second, the amplification-based signal creates a fundamental conflict: perturbations that effectively fool the classifier inevitably become more detectable after sanitization. Together, these constraints make it practically impossible for an adaptive adversary to simultaneously maintain high attack success and low detectability.

\vspace{-3mm}
\section{Related Work and Discussions}
\vspace{-1mm}
\textbf{Adversarial Attacks.}
Adversarial examples are generated by optimizing small perturbations that maximize the target model’s loss. White box attacks such as FGSM \cite{fgsm}, PGD \cite{fgsm}, and APGD \cite{apgd} directly exploit gradient information to craft effective perturbations. BIM \cite{bim} improves upon these gradient based methods through iterative updates. Beyond white box access, transferability based black box attacks \cite{vmi_vni, mumcu2024sequential} generate adversarial examples using substitute models and exploit the cross model generalization of perturbations. 
AutoAttack \cite{aa} combine multiple attack types to comprehensively evaluate model robustness under diverse conditions.

\vspace{-1mm}

\textbf{Adversarial Detection Methods.}
Early detection methods focused on input level cues, such as feature squeezing \cite{fs}, bit depth reduction, or spatial smoothing \cite{denoise}. Subsequent approaches explored semantic or reconstruction based detection \cite{deflection}, leveraging inconsistencies in feature representations or reconstruction fidelity to distinguish adversarial examples. Consistency based detectors \cite{vlad, vilas} instead rely on prediction stability under input transformations.
Of particular relevance to our framework are the methods that explicitly exploit the nonuniform impact of adversarial perturbations across network layers. Deep kNN \cite{dknn} performs nearest neighbor classification on intermediate features to identify outlier activations, though it requires high memory and computational cost. Nearest Neighbors Influence Functions (NNIF) \cite{nnif} extend this concept by combining layer-wise influence estimation with k-nearest neighbor analysis in the feature space to measure how each training sample affects model predictions. While effective, NNIF relies on precomputed adversarial examples and incurs prohibitive computational cost, with the authors explicitly noting that it is unsuitable for real-time applications. Layer Regression (LR) \cite{lr} adopts a lightweight design by training a multilayer perceptron (MLP) to learn the regression consistency between adjacent layers, thereby quantifying representational stability and revealing how adversarial noise perturbs internal activations unevenly. These methods collectively demonstrate that adversarial examples introduce structured disruptions within deep networks, a phenomenon central to our theoretical formulation. In contrast, our proposed JAD detector captures this same asymmetry implicitly through the network’s statistical response to JPEG-based sanitization, without requiring auxiliary models, stored references, adversarial training samples, or layer-wise computations.

\vspace{-1mm}
\textbf{Singular Value Regularization.}
Singular value regularization has been well studied, mostly from the perspective of limiting the spectral norm or Lipschitz constant, which corresponds to minimizing the largest singular values of weight matrices \cite{tsuzuku2018lipschitz,farnia2018generalizable,wu2024lrs}. There are also several studies which propose to regularize all singular values within a narrow band to obtain small condition numbers close to one \cite{cisse2017parseval,jia2017svb,senderovich2022practical,xiao2018dynamical,jere2020singular}. The existing works discuss singular value regularization in terms of training stability, generalization performance, and improved robustness to adversarial samples. We are not aware of any existing work that regularizes singular values of weight matrices for adversarial sample detection. 

\vspace{-3mm}
\section{Conclusion}
\vspace{-2mm}
We established a formal adversarial noise amplification theorem that provides the first rigorous explanation for a phenomenon widely observed in deep neural networks. Building on this theoretical basis, we designed a network architecture and a spectral regularization strategy that consistently strengthen amplification, creating a clear separation between adversarial and natural perturbations. Using this property, we introduced JAD, a simple and lightweight detector that requires no auxiliary models or adversarial training and operates entirely at inference time. Extensive experiments show that JAD maintains exceptional robustness even against fully adaptive adversaries, highlighting amplification as a powerful and practical signal for adversarial defense.

\setlength{\bibsep}{5.5pt}
\bibliography{main}
\bibliographystyle{plainnat}

\newpage
\appendix
\onecolumn

\clearpage

\section{Additional Implementation Details}
\label{app:implement}

In Sections~5.1, 5.2, and~5.4, all adversarial attacks are evaluated using the configurations originally proposed in their respective papers. We implement these attacks using the {Torchattacks} library.

For experiments involving {WideResNet}, we use a network depth of 14. In our configuration, we replace all ReLU activations with LeakyReLU and apply the loss function introduced in Section~3, setting $\lambda = 10^{-2}$ while keeping all other settings identical to the default implementation. For both models, we use the Adam optimizer with a learning rate of $10^{-2}$.

\section{Additional Experiments with JAD}
\label{app:imagenet}
\begin{table}[t]
\centering
\setlength{\tabcolsep}{4pt}
\begin{tabular}{l|ccccccc}
\hline
\textbf{Method} & \textbf{PGD} & \textbf{BIM} & \textbf{VNI} & \textbf{VMI} & \textbf{PGD-$\ell_2$} & \textbf{AA} & \textbf{Diff} \\ \hline
ResNet-50     & 0.99 & 0.99 & 0.99 & 0.99 & 0.99 & 0.99 & 0.99 \\
ViT           & 0.99 & 0.99 & 0.99 & 0.99 & 0.99 & 0.98 & 0.98 \\ \hline
Robust Model 1 & 0.97 & 0.97 & 0.94 & 0.95 & 0.98 & 0.95 & 0.96 \\
Robust Model 2 & 0.98 & 0.98 & 0.97 & 0.98 & 0.98 & 0.96 & 0.96 \\
Robust Model 3 & 0.97 & 0.97 & 0.95 & 0.96 & 0.97 & 0.94 & 0.96 \\ \hline
\end{tabular}
\caption{Performance of JAD across standard (ResNet-50, ViT) and robust models under multiple attacks.}
\label{tab:jad_combined_models}
\vspace{-3mm}
\end{table}

In this section, we extend our evaluation to the ImageNet validation set using 10,000 randomly selected images under the same experimental protocol as our main results. We consider both standard architectures, ResNet-50 and ViT, as well as three representative robust models drawn from the RobustBench benchmark~\cite{rbench}, denoted as Robust Model~1~\cite{robust1}, Robust Model~2~\cite{robust2}, and Robust Model~3~\cite{robust3}. The standard ResNet-50 and ViT models serve as reference points for comparison. As shown in Table~\ref{tab:jad_combined_models}, the proposed JAD detector maintains consistently high AUROC across all attacks on both standard and robust models, indicating that the underlying amplification signal persists even under robustness-enhancing training procedures.

Robustness has been extensively studied with the goal of reducing model sensitivity to adversarial perturbations \cite{robustsurvey1, robustsurvey2, mumcu2026robustness}, which may initially appear at odds with our amplification-based perspective. However, our formulation does not contradict robustness. In particular, Assumption~1 does not require maximizing the lower-bound constant $L$ or enforcing large amplification. It only requires a strictly positive lower bound, i.e., $L>0$. Therefore, even a very small $L$, as in common activation functions such as Leaky ReLU, is sufficient to satisfy the assumption. Moreover, an activation function can be both Lipschitz smooth and satisfy our assumption. Specifically, a function $f_i$ can be lower and upper bounded as
\[
L_1\|u-v\|_2 \leq \|f_i(u)-f_i(v)\|_2 \leq L_2\|u-v\|_2,
\qquad L_2 > L_1 > 0.
\]
Thus, Assumption~1 does not imply that $f_i$ is non-Lipschitz, nor does it enforce non-smoothness. Instead, it imposes a mild nonzero lower-bound condition that can coexist with robustness-oriented constraints.

Together, these results highlight the generality of our approach and demonstrate its effectiveness beyond standard, non-robust settings.

\section{Performance of standalone JPEG as a detector}

\label{app:jpeg_standalone}

\begin{table}[t]
\centering
\begin{tabular}{l|ccccc}
\hline
\textbf{Method} & \textbf{PGD} & \textbf{BIM} & \textbf{VNI} & \textbf{VMI} & \textbf{PGD-{$\ell_2$}} \\ \hline
JPEG   & 0.60 & 0.59 & 0.55 & 0.54 & 0.61 \\
JAD    & 0.99 & 0.99 & 0.99 & 0.98 & 0.99 \\ \hline
\end{tabular}
\caption{Comparison of AUROC scores between standalone JPEG and JAD.}
\label{tab:jpeg}
\end{table}

We evaluate JPEG sanitization as a standalone detector by comparing the model’s predictions before and after JPEG compression. If the predictions remain consistent, we classify the input as benign; a change in prediction indicates an adversarial example. Using the same experimental settings described in Section~5.4, we construct a balanced evaluation set consisting of 1,000 clean CIFAR-10 test images and 1,000 corresponding adversarial examples that successfully induce misclassification. The resulting AUROC scores are reported in Table~\ref{tab:jpeg}.

As shown in Table~\ref{tab:jpeg}, standalone JPEG-based detection performs poorly across all attacks, with AUROC values around 0.55--0.61. In contrast, our proposed JAD method achieves near-perfect detection performance (0.98--0.99) under all attack settings. These results highlight that simple input sanitization is insufficient for reliable adversarial detection, while JAD provides a substantially more robust and consistent detection signal.

\section{Adaptive Attack Trial Number Analysis}
\label{app:adaptive}

We further analyze the effect of increasing the number of JPEG trials $T$ in the adaptive attack. In this setting, we fix the attack parameters to $\epsilon = 8$, step size $=200$, and the adaptive weight $\lambda = 0.5$. Because each trial requires a full PGD pass, the total attack complexity scales linearly as
\[
\text{Complexity} = T \times \text{PGD step size} = 200T.
\]

As shown in Figure~\ref{fig:ablation_trial}, varying $T$ from 1 to 20 results in only negligible changes in detection performance. The AUROC remains consistently high fluctuating by roughly one percentage point, indicating that increasing $T$ offers almost no benefit to the adversary in terms of bypassing JAD.

In contrast, the attack complexity increases sharply and linearly with $T$, exceeding 4{,}000 gradient steps when $T=20$. This demonstrates that while larger $T$ values dramatically increase computational cost, they do not meaningfully improve the adaptive attack’s effectiveness.

\begin{figure}[t]
    \centering
    \includegraphics[width=0.5\linewidth]{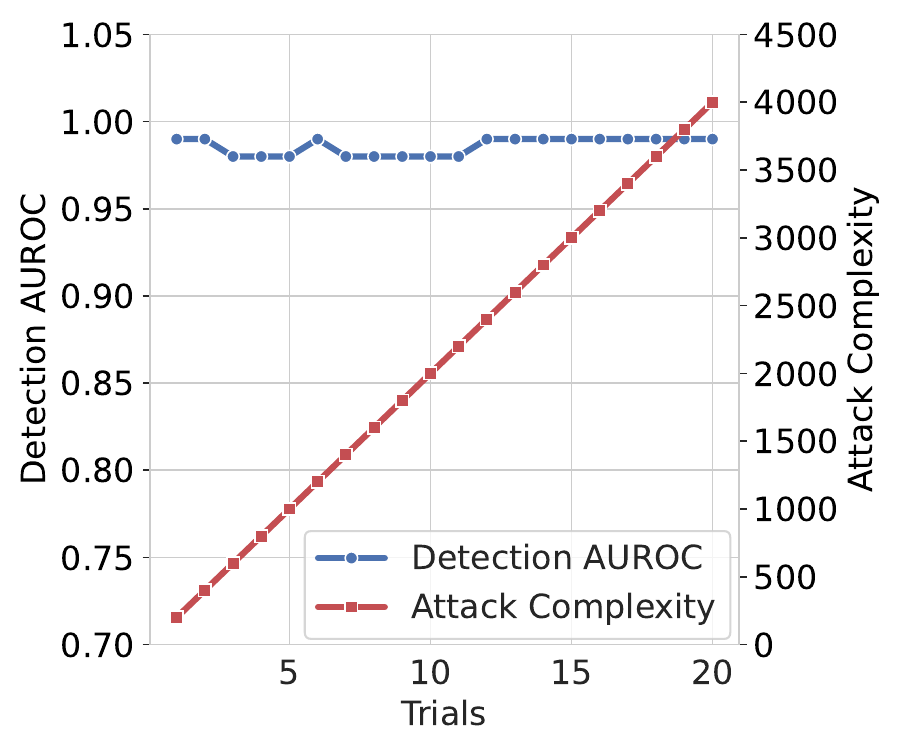}
    \caption{Detection AUROC and attack complexity as the number of JPEG trials $T$ increases from 1 to 20. AUROC remains highly stable with only minor fluctuations, while attack complexity increases linearly with $T$.}
    \label{fig:ablation_trial}
\end{figure}

\section{Broader Impact Statement}

This paper presents theoretical and empirical work aimed at advancing the robustness and reliability of machine learning models in the presence of adversarial perturbations. By formalizing adversarial noise amplification and proposing a lightweight detection mechanism, the work contributes to the development of more trustworthy systems, particularly in settings where adversarial manipulation may pose risks.

The proposed methods are defensive in nature and focus on detecting adversarial inputs at inference time. While improved understanding of adversarial behavior can inform both attacks and defenses, this work does not introduce new mechanisms for generating adversarial examples. We do not anticipate significant negative societal or ethical impacts beyond those already inherent in research on adversarial robustness.

\section{Pseudo Codes for the Experiments}
\label{app:code}

We compute the net amplification signal $\beta$ using Algorithm \ref{alg:dnd1}. The pseudo-code for the adaptive attack targeting JAD is given in Algorithm \ref{alg:adaptive_attack}.

\begin{algorithm}[h]
\caption{Measuring the Net Amplification Ratio}
\label{alg:dnd1}
\begin{algorithmic}[1]

\Require Model $g$, clean image $x$, true label $y$, Attack Function $A$
\Ensure Net Amplification Ratio $\beta$

\Statex
\Function{NetAmplification}{$g, x, y, A$}

    \State $x^{\text{adv}} \gets A(g, x, y)$ 
        \Comment{Generate the adversarial example}
    
    \State $Z \gets g.\text{get\_all\_activations}(x)$ 
        \Comment{Get activations for all $n$ layers}
    \State $Z^{\text{adv}} \gets g.\text{get\_all\_activations}(x^{\text{adv}})$
    
    \State Initialize list $Impacts$
    \For{$i = 1$ to $n$}
        \State $d_i \gets \| Z^{\text{adv}}[i] - Z[i] \|_{2}$ 
            \Comment{Calculate L2-distance}
        \State Append $d_i$ to $Impacts$
    \EndFor
    
    \State $d_1 \gets Impacts[1]$
    \State $d_n \gets Impacts[n]$
    
    \If{$d_1 > 0$}
        \State $\beta \gets d_n / d_1$
    \Else
        \State $\beta \gets 0$
    \EndIf
    
    \State \Return $\beta$
\EndFunction

\end{algorithmic}
\end{algorithm}

\begin{algorithm}[h]
\caption{Adaptive attack against JAD}
\label{alg:adaptive_attack}
\begin{algorithmic}[1]

\Require Classifier $g$, clean image $x$, true label $y$
\Require Perturbation budget $\epsilon$, step size $\alpha$, iterations $K$
\Require EOT trials $T$, balancing weight $\lambda$, quality range $Q$
\Ensure Adversarial example $x^{\text{adv}}$

\Statex
\Function{AdaptiveJAD}{$g, x, y, \epsilon, \alpha, K, T, \lambda, Q$}

    \State $x^{\text{adv}} \gets x + \mathcal{U}(-\epsilon, \epsilon)$
        \Comment{Random initialization}
    \State $x^{\text{adv}} \gets \text{Clamp}(x^{\text{adv}}, 0, 1)$

    \For{$k = 1$ to $K$}

        \State $\mathcal{L} \gets \mathcal{L}_{\text{CE}}(g(x^{\text{adv}}), y)$
            \Comment{Loss on current adversarial example}
        \State $f \gets \nabla_{x^{\text{adv}}} \mathcal{L}$

        \State Initialize $f_{\text{robust}} \gets 0$
            \Comment{Expected gradient over JPEG-sanitized inputs}
        \For{$t = 1$ to $T$}
            \State Sample $q \sim \mathcal{U}(Q)$
            \State $x^{\text{san}} \gets \text{JPEG}_{\text{STE}}(x^{\text{adv}}, q)$
                \Comment{JPEG with straight-through estimator}
            \State $\mathcal{L}^{\text{san}} \gets \mathcal{L}_{\text{CE}}(g(x^{\text{san}}), y)$
            \State $f_{\text{robust}} \gets f_{\text{robust}} +
                   \nabla_{x^{\text{adv}}} \mathcal{L}^{\text{san}}$
        \EndFor
        \State $f_{\text{robust}} \gets f_{\text{robust}} / T$

        \State $f_{\text{total}} \gets f + \lambda \cdot f_{\text{robust}}$
            \Comment{Combine gradients}
        \State $x^{\text{adv}} \gets x^{\text{adv}} + \alpha \cdot \text{sign}(f_{\text{total}})$

        \State $x^{\text{adv}} \gets \text{Clamp}(x^{\text{adv}}, x - \epsilon, x + \epsilon)$
            \Comment{Project to $\ell_\infty$-ball}
        \State $x^{\text{adv}} \gets \text{Clamp}(x^{\text{adv}}, 0, 1)$
            \Comment{Clip to valid pixel range}
    \EndFor

    \State \Return $x^{\text{adv}}$
\EndFunction

\end{algorithmic}
\end{algorithm}

\end{document}